\documentclass{article}
\usepackage[nonatbib, final]{neurips_2020_tda}
\usepackage[utf8]{inputenc}
\usepackage[T1]{fontenc}
\usepackage{siunitx}
\usepackage{hyperref}  
\usepackage{wrapfig}
\usepackage{url}            
\usepackage{booktabs}       
\usepackage{amsfonts}       
\usepackage{nicefrac} 
\usepackage{subcaption}
\usepackage{graphicx}
\usepackage{microtype}  
\usepackage{amsmath, amsthm}
\usepackage[font=small]{caption}
\theoremstyle{definition}

\newtheorem{theorem}{Theorem}[section]

\title{Permutation invariant networks to learn Wasserstein metrics}
\author{Arijit Sehanobish \\
  Internal Medicine (Cardiology) and Computer Science\\
 Yale University\\
  \texttt{arijit.sehanobish@yale.edu}\\
\and
Neal G. Ravindra \\
Internal Medicine (Cardiology) and Computer Science\\
 Yale University\\
  \texttt{neal.ravindra@yale.edu}\\
  \and 
  David van Dijk\\
  Internal Medicine (Cardiology) and Computer Science\\
 Yale University\\
  \texttt{david.vandijk@yale.edu}
  }

\begin{document}

\maketitle

\begin{abstract}

Understanding the space of probability measures on a metric space equipped with a Wasserstein distance is one of the fundamental questions in mathematical analysis. The Wasserstein metric has received a lot of attention in the machine learning community especially for its principled way of comparing distributions. In this work, we use a permutation invariant network to map samples from probability measures into a low-dimensional space such that the Euclidean distance between the encoded samples reflects the Wasserstein distance between probability measures. We show that our network can generalize to correctly compute distances between unseen densities. We also show that these networks can learn the first and the second moments of probability distributions. 
    
\end{abstract}

\section{Introduction}

The Wasserstein distance is a distance function between probability measures on a metric space $\mathcal{X}$. It is a natural way to compare the probability distributions of two variables $X$ and $Y$, where one variable is derived from the other by small, non-uniform perturbations, while strongly reflecting the metric of the underlying space $\mathcal{X}$. It can also be used to compare discrete distributions. The Wasserstein distance enjoys a number of useful properties, which likely contributes to its wide-spread interest amongst mathematicians and computer scientists~\cite{Bobkov2019OnedimensionalEM, spaceprob, bigot2013geodesic, canas2012learning, delbarrio1999, givens1984, VillaniOT1, VillaniOT2, arjovsky2017wasserstein}. However, despite it's broad use, the Wasserstein distance has several problems. For one, it is computationally expensive. Second, the Wasserstein distance is not Hadamard differentiable, which can present serious challenges when trying to use it in machine learning. Third, the distance is not robust. To alleviate these problems, one can use various regularized entropies to compute an approximation of this Wasserstein distance. Such an approach is more tractable and also enjoys several nice properties~\cite{altschuler2018nearlinear, cuturi2013sinkhorn, peyre2020computational}. 

In this short article, we are interested in learning about the Wasserstein space of order $p$, i.e. an infinite dimensional space of all probability measures with up to $p$-th order finite moments on a complete and separable metric space $\mathcal{X}$. More specifically, we asked,  (\textbf{1}) can we propose a neural network that correctly computes the Wasserstein distance between $2$ measures, even if both of them are not in our training examples? (\textbf{2}) What properties of the measures does such a network learn? For example, does it learn something about the moments of these distributions? (\textbf{3}) What properties of the original Wasserstein space can we preserve in our encoded space? 

There has been a lot of work in understanding the space of Gaussian processes~\cite{2wasnipsgauss, takatsu2011} but our work is more similar to, which attempts to understand Wasserstein spaces with neural networks~\cite{courty2017learning, frogner2019learning}. Like~\cite{courty2017learning}, we use a Siamese network to compare and contrast various densities but the questions we address in this article are different than that of~\cite{courty2017learning}. Furthermore, we try to approximate the Wasserstein space by learning a mapping from the space to a low dimensional Euclidean space, unlike~\cite{frogner2019learning}, where they learn a mapping from an Euclidean space to the Wasserstein space.
\section{Theory} Let $\mathcal{X}$ be a complete and separable metric space. For simplicity, we take $\mathcal{X}$ to be $\mathbb{R}^{n}$ or a closed and bounded subset of $\mathbb{R}^{n}$. Let $\mathbb{P}(\mathcal{X})$ be the space of all probability measures on $\mathcal{X}$. One can endow the space $\mathbb{P}(\mathcal{X})$ with a family of metrics called the Wasserstein metrics $W_p$.  
\begin{equation}\label{eqn:ot}
    W_p(\mu, \nu) = \text{inf}_{\substack{X\sim \mu \\ Y\sim\nu}} \mathbb{E}(|X-Y|^p)^{1/p}, \ \ \ p\geq 1
\end{equation}
We use the notations $W_p(X,Y)$ and $W_p(\mu,\nu)$ interchangeably whenever $X\sim \mu$ and $Y \sim \nu$. We also assume that $\mathbb{E}(|X|^p)$ (and $\mathbb{E}(|Y|^p)$) is finite. Most of the following properties regarding the space $\mathbb{P}(\mathcal{X})$ and $W_p$ are well-known but we summarize them for the convenience of the reader~\cite{Santambrogio2015, statwass}. 
\begin{theorem}\label{wass_prop} (\textbf{i}) $\mathbb{P}(\mathcal{X})$ equipped with $W_p$ is a complete and separable metric space.\\
(\textbf{ii}) If $X$ and $Y$ are degenerate at $x,y \in \mathcal{X}$, then $W_p(X,Y) = |x-y|$.\\
(\textbf{iii}) (Scaling law) For any $a \in \mathbb{R}$, $W_p(aX,aY) = |a|W_p(X,Y)$. \\
(\textbf{iv}) (Translation invariance) For any $x \in \mathcal{X}$, $W_p(X+x,Y+x) = W_p(X,Y)$ \\
(\textbf{v}) $\mathbb{P}(\mathcal{X})$ is flat metric space under $W_1$ and $W_2$ if $X= \mathbb{R}$ but the sectional curvature is non-negative under $W_2$.
\end{theorem}
\begin{proof}
See~Section 2 and Section $4$ in \cite{statwass}. 
\end{proof}
For a class of random variables on a compact set $\mathcal{X}$, all metrics $W_p$, $p > 1$ are equivalent. 
\begin{theorem}(Topology generated by $W_p$) (\textbf{i}) If $\mathcal{X} \subset \mathbb{R}^n$  is compact and $p \in [1,\infty)$, in the space $\mathbb{P}(\mathcal{X})$, we have $\mu_k \rightarrow \mu$  iff $W_p (\mu_k , \mu) \rightarrow 0$. \\
(\textbf{ii}) If $\mathcal{X} = \mathbb{R}^n$, then $W_p (\mu_k , \mu) \rightarrow 0$ iff $\mu_k \rightarrow \mu$  and $\int |x|^p d\mu_k \rightarrow \int |x|^p d\mu$
\end{theorem}
\begin{proof}
See proofs associated with Theorem $5.10$ and Theorem $5.11$ in~\cite{Santambrogio2015}.
\end{proof}

The measures $\mu$ and $\nu$ are rarely known in practice. Instead, one has access to finite samples $\{x_i\} = X \sim \mu$ and $\{y_j\}= Y \sim \nu$. We then construct discrete measures $\mu := \sum_{i=1}^{n}   a_i\delta_{x_{i}}$ and  $\nu := \sum_{j=1}^{m}b_j\delta_{y_{j}}$ where $a, b$ are vectors in the probability simplex, and the pairwise costs can be compactly represented as an $n \times m$ matrix $C$, i. e., $c_{ij} := c(x_i, y_j)$ where $c$ is the metric of the underlying space $\mathcal{X}$. Since the marginals here are fixed to be the laws of $X$ and $Y$ , the problem is to find a copula~\cite{Sklar1959FonctionsDR} that couples X and Y together as “tightly” as possible in an $L^p$-sense, on average; if $p = 2$ then the copula one seeks is the one that maximizes the correlation (or covariance) between $X$ and $Y$ , i.e., the copula inducing maximal linear dependence. Solving the above problem scales cubically on the sample sizes and is extremely difficult in practice. Adding an entropy regularization,
leads to a problem that can be solved much more efficiently~\cite{altschuler2018nearlinear,cuturi2013sinkhorn, peyre2020computational}. In this article, we use the Sinkhorn distance $SD^{\lambda}_{p}$ and their computation, as in~\cite{cuturi2013sinkhorn}. For more details about the entropic regularization, please see Appendix C. The Sinkhorn distance however is not a true metric~\cite{cuturi2013sinkhorn} and fails to satisfy $SD^{\lambda}_p(X,X) = 0$. The technical workaround this issue is explained in Appendix D. Moreover, the Sinkhorn distance requires discretizing the space, which alters the metric. The goal of this paper is to see how well a neural network, trained using the Sinkhorn distance, can capture the topology of the $\mathbb{P}(\mathcal{X})$ under $W_p$.
\begin{figure}[h]
  \centering
  \includegraphics[width=.85\linewidth]{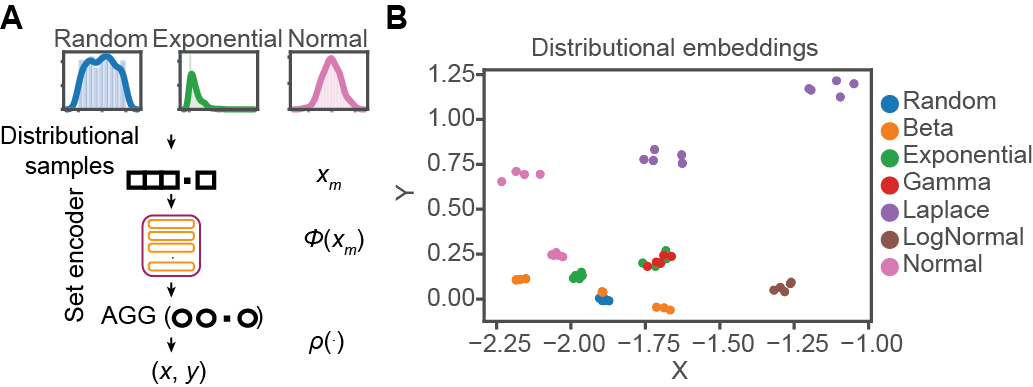}
  \caption[0.85\textwidth]{(\textbf{A}) Our distributional encoder. (\textbf{B}) Low-dimensional embedding of encoded distributions.}
  \label{fig:modelfig}
\end{figure}

\section{Neural Networks to understand \texorpdfstring{$\mathbb{P}(\mathcal{X})$}{Lg}}

We draw random samples with replacement of size $N$ from various distributions in $\mathbb{P}(\mathcal{X})$. For technical reasons (described in Appendix D), we only use continuous distributions during the training process. We use the DeepSets architecture~\cite{zaheer2017deep} to encode this set of $N$ elements as we want an encoding that is invariant of the permutations of the samples. More precisely, if $X \sim \mu$ and $Y \sim \nu$, ($\mu =\nu$ is allowed, but $X$ and $Y$ are drawn independently) and we denote the set of samples drawn from $\mu$ as $S_X$ (similarly $S_{\nu}$), we train the encoder $H_{\theta}$ such that, 
\begin{equation}
    ||H_{\theta}(S_X)-H_{\theta}(S_Y)|| = SD^{\lambda}_{p}(\mu,\nu)
\end{equation}
Thus, the loss function becomes, 
\begin{equation}
    L_{wass} = \frac{1}{\binom {m}{2}} \sum (||H_{\theta}(S_X)-H_{\theta}(S_Y)|| - SD^{\lambda}_{p}(\mu,\nu))^2
\end{equation}
where $m$ is the size of the mini-batch and we pick $2$ sets at random from the mini-batch to compare distances. One can think of our network as a Siamese Network~\cite{Koch2015SiameseNN} with a DeepSet backbone which allows us to compare and contrast samples drawn from same or different distributions. Our work can be thought as \textit{next-generation functional data analysis}~\cite{wang2015review} (Section 6). 
More details about the network architecture can be found in Appendix B. The code is available at \url{https://github.com/arijitthegame/encoding_wasserstein_metrics}.

\subsection{Regularizers for ensuring better properties}
If $X'= X+x$, then $S_{X'}$ is a  set of samples $X$ after translation $x$ (this similarly applies for $Y'$ and $S_{Y'}$). To ensure the properties of $W_p$ are reflected in our computed Euclidean distance, we demand that, 
\begin{enumerate}
    \item $||H_{\theta}(S_X')-H_{\theta}(S_Y')|| = ||H_{\theta}(S_X)-H_{\theta}(S_Y)||$ 
    \item $||H_{\theta}(S_{aX})-H_{\theta}(S_{aY})|| = |a|||H_{\theta}(S_X)-H_{\theta}(S_Y)||$.
\end{enumerate}

These constraints comprise the loss function 
\begin{multline*}
\mathcal{L} : = L_{Wass} + \frac{1}{\binom {m}{2}} \sum((||H_{\theta}(S_X')-H_{\theta}(S_Y')|| -  ||H_{\theta}(S_X)-H_{\theta}(S_Y)||)^2 \\
     + (||H_{\theta}(S_{aX})-H_{\theta}(S_{aY})|| - |a|||H_{\theta}(S_X)-H_{\theta}(S_Y)||))^2
\end{multline*}
\section{Experiments}
In this section we will describe our toy examples and show the discriminative behavior of our Neural Networks and the interesting properties of the space it can uncover. Our datasets are the following: \textbf{(1)} Random samples of size $500$ drawn independently about $50$ times from uniform, Normal, Beta, Gamma, Exponential, Laplace, Log Normal and mixtures of Gaussian distributions with varying parameters. \textbf{(2)} Random samples of size $300$ drawn independently about $100$ times from $2D$ Normal distributions with various $\mu, \Sigma$.
Fig~\ref{fig:modelfig} (\textbf{B}) shows the embedding our datasets by our model. In Fig~\ref{fig:correlationplot}, we show how well the neural network approximates the Sinkhorn distances from samples drawn from our test densities. All the results shown here are with the $W_1$ metric. Other plots showing how well our network respects the scaling law (\textbf{iii}) in Theorem~\ref{wass_prop} and the results with the $W_2$ metric are shown in the Appendix A. Detailed quantitative results can be found in Appendix F. Some of our results with the $W_2$ metric are weaker than the ones with the $W_1$ metric. This may be due to the following reasons: $\mathbb{P}(\mathcal{X})$ under $W_2$ is no longer flat and the Sinkhorn distance changes the metric differently than it changes the space under $W_1$; secondly, since our target is an Euclidean space which is a flat space, we are losing more structural information when mapping from the $2$-Wasserstein space. \\
\begin{figure}[h!]
  \centering
  \includegraphics[width=\linewidth]{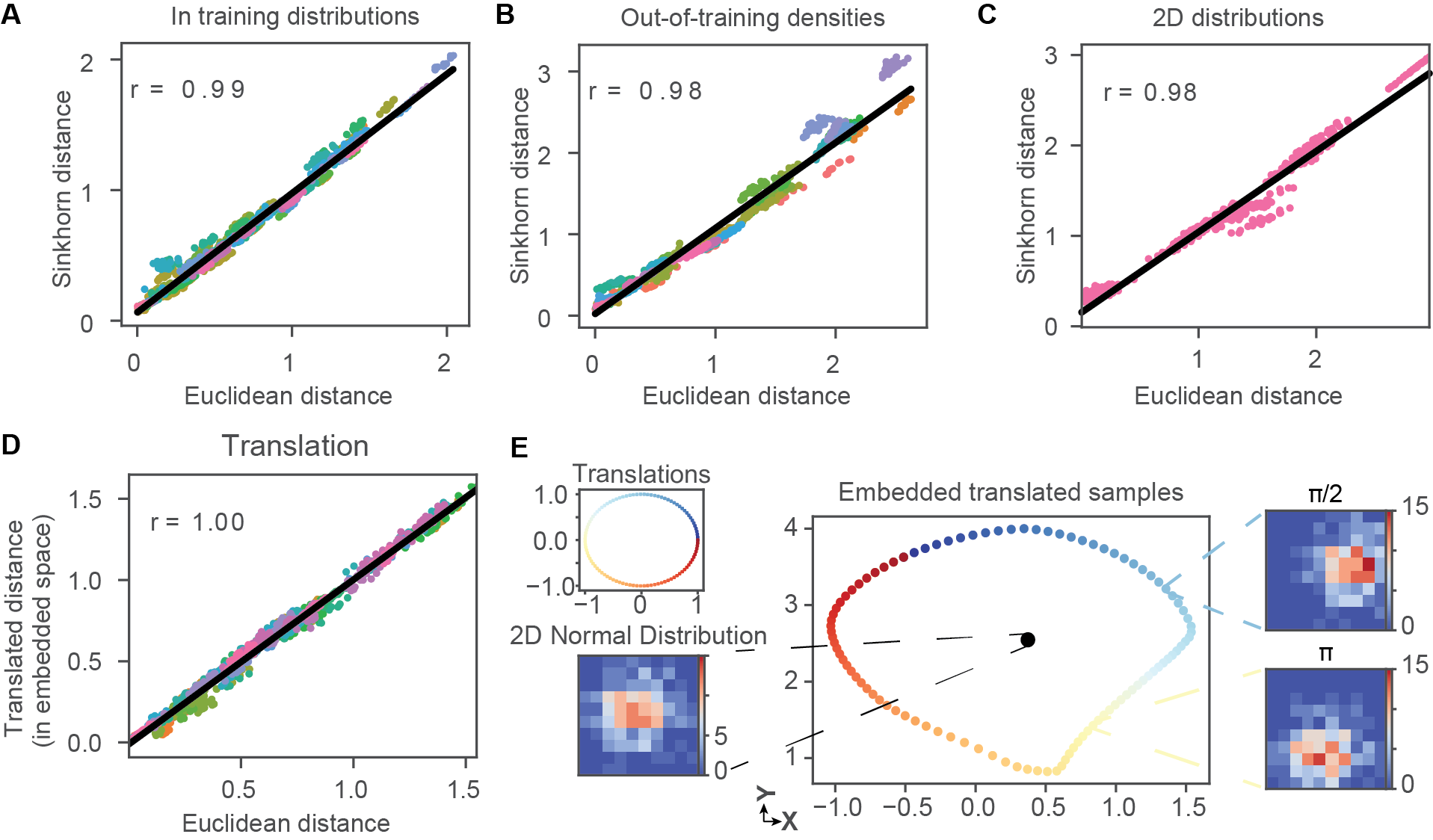}
  \caption[0.5\textwidth]{(\textbf{A}--\textbf{C}) Pearson's r correlation coefficient for  association between embedded and Sinkhorn distances (color code in Appendix A). (\textbf{D}) Correlation after translations. (\textbf{E}) Samples from a multivariate normal distribution translated around a circular path.}
  \label{fig:correlationplot}
\end{figure}
\textbf{Generalizing to out-of-sample-densities :} We also show that our model can generalize well to densities that are out of our training set. These densities are primarily constructed from the training densities but by changing the parameters (Fig~\ref{fig:correlationplot} \textbf{B,C}). But even more interestingly, we found that our model can correctly measure the distance between $2$ Dirac measures and distance between $2$ Binomial densities, even though they are not a part of the training densities. \\
\textbf{Translating samples :} Given $2$ samples $X \sim \mu, Y \sim \nu$, we can translate them around by a random vector $a$, to create new samples $X':= X + a, Y':= Y + a$, under property $4$ in Theorem $3.1$, $H_{\theta}(X), H_{\theta}(Y), H_{\theta}(X'), H_{\theta}(Y')$ would form a parallelogram. Fig~\ref{fig:correlationplot}(\textbf{D}) shows the exact relationship between the distances of encoded translated samples and the encoded samples.
Furthermore, we took samples from a $2D$ Normal Distribution $N(\mu, \Sigma)$ and rotated it around by using a circle, i.e. created new samples via $X' := X + (\text{cos}(\phi), \text{sin}(\phi))$ and we found that the encoded translated samples also formed a circular pattern around the original encoded sample. Thus our simple examples show that our metric preserves the translation invariance property and some geometry of the space (Figure~\ref{fig:correlationplot}E). \\
\textbf{Learning statistical properties of the measures :} Surprisingly for encoded $1D$-distributions, we found the strong correlation between means (and variances) of the distributions and the $x$-coordinate (and y-coordinate) of the encoded point( Fig~\ref{fig:ax_barycenter}\textbf{A,B}). That explains why the encoded Dirac distribution at $0$ and Normal distribution with mean $0$ and standard deviation $\sigma$ has $x$-coordinates close to each other. An open question and an interesting future work will to be understand if we can capture higher moments as we increase the output dimension. \\
\textbf{Respecting the topology of the space :}
We know that the Dirac delta measure is the limit of Gaussian measures under the weak convergence of measures. Choosing samples drawn from $N(0, 1/n)$ we can see that our encoded points converge to the point encoded by the Dirac measure (Fig~\ref{fig:ax_barycenter}\textbf{C}). This gives us an empirical evidence that our neural network may be continuous with respect to the Wasserstein metric.\\
\textbf{Wasserstein barycenters :} Given two densities $\mu_1, \mu_2$, if $\hat{\mu}$ is their Wasserstein barycenter~\cite{anderes2015discrete, barycenterwass, zemel2017frchet, karchermean, karcherriemannian}, our aim is to show that $H_{\theta}(\hat{\mu})$ can be approximated by the midpoint of the line joining $H_{\theta}(\mu_1)$ and $H_{\theta}(\mu_2)$. Fig~\ref{fig:ax_barycenter}\textbf{D} shows the following examples of this claim: \textbf{1)} Samples drawn from $N(0,.1)$ and $N(1,.1)$. \textbf{2)} Dirac at $0$ and $1$. \textbf{3)} Uniform distribution in $[0,.1]$ and in $[.8,.1]$. 

\begin{figure}[ht!]
  \centering
  \includegraphics[width=\linewidth]{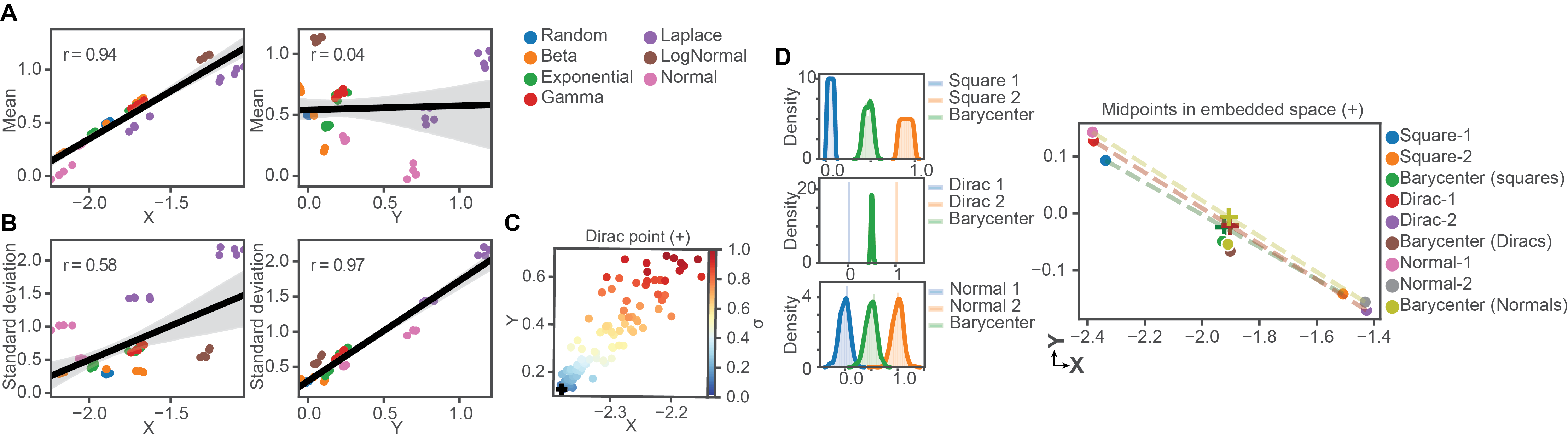}
  \caption[0.01\textwidth]{Person's r comparing embedding axes to means (\textbf{A}) and standard deviations (\textbf{B}). (\textbf{C}) Convergence of samples from Normal distributions with various standard deviations to the Dirac distribution encoding. (\textbf{D}) Barycenters of  distributions (left) and midpoints drawn between lines connecting the encoded samples (right).}
  \label{fig:ax_barycenter}
\end{figure}
We also note that none of the measures used above are in the set of our training measures. And finally observe that the figure also shows the correlation between $x$-coordinates and means of the chosen measures. Finally, the experiment also show that we can approximate the Wasserstein geodesic by straight lines in our encoded space. 

\section{Effect of the regularizers}
In this section we show the effect of the regularizers in generalizing to unseen distributions and also how well our model learns the Translation law. We show that by just adding the regularizer that enforces the Scaling law, we get better performance in learning the Translation property and generalization to out of sample distributions than the vanilla model which has no regularizers.  
\begin{figure}[h!]
  \centering
  \includegraphics[width=\linewidth]{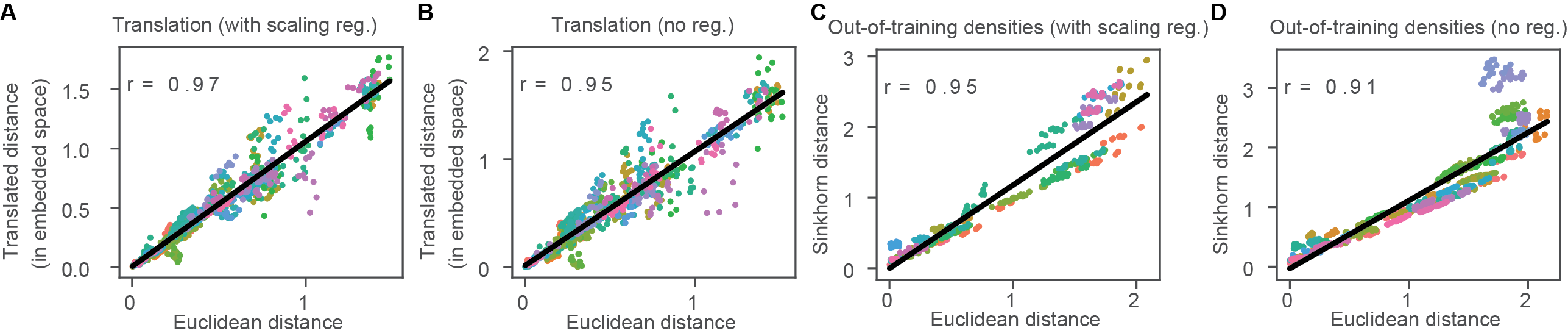}
  \caption[0.01\textwidth]{Pearson's r correlation to compare distances using model trained with (\textbf{A}) scaling regularizer, (\textbf{B}) no regularizer. Out of sample distribution generalization plots for models trained with (\textbf{C}) scaling regularizer, (\textbf{D}) no regularizer. }
  \label{fig:ablation}
\end{figure}

\section{Conclusion and Future Work}
In this work we showed that we learned a metric by approximating the Wasserstein distance by Sinkhorn distance that obeys the translation invariance and also generalizes to some unseen measures. For $1D$ measures, we found strong correlation between the encoded vectors $x$ coordinates (resp. y-coordinates) with means (variance) of the samples. We are excited by these toy results and would like to prove continuity properties of our neural network. We have also shown our model in general performs better when trained with $W_1$ metric and we would like to understand what role does the topology of the Wasserstein space play in the difference in performances of our model. Finally we would like to investigate if our model can learn higher moments as we increase the output dimension and finally to quantify the distortion of the original Wasserstein metric by our embedding.

\section*{Acknowledgements}
The first author wants to thank Alexander Cloninger for helpful suggestions and for suggesting to study the geometry of the Wasserstein space by simple translations and scalings. The authors would also like to thank the anonymous reviewers for helpful comments and suggestions. 

\nocite{*}
\bibliographystyle{unsrt}
\bibliography{ref}
\newpage
\appendix
\section{Additional Figures}
\begin{figure}[t]
\includegraphics[width=\linewidth]{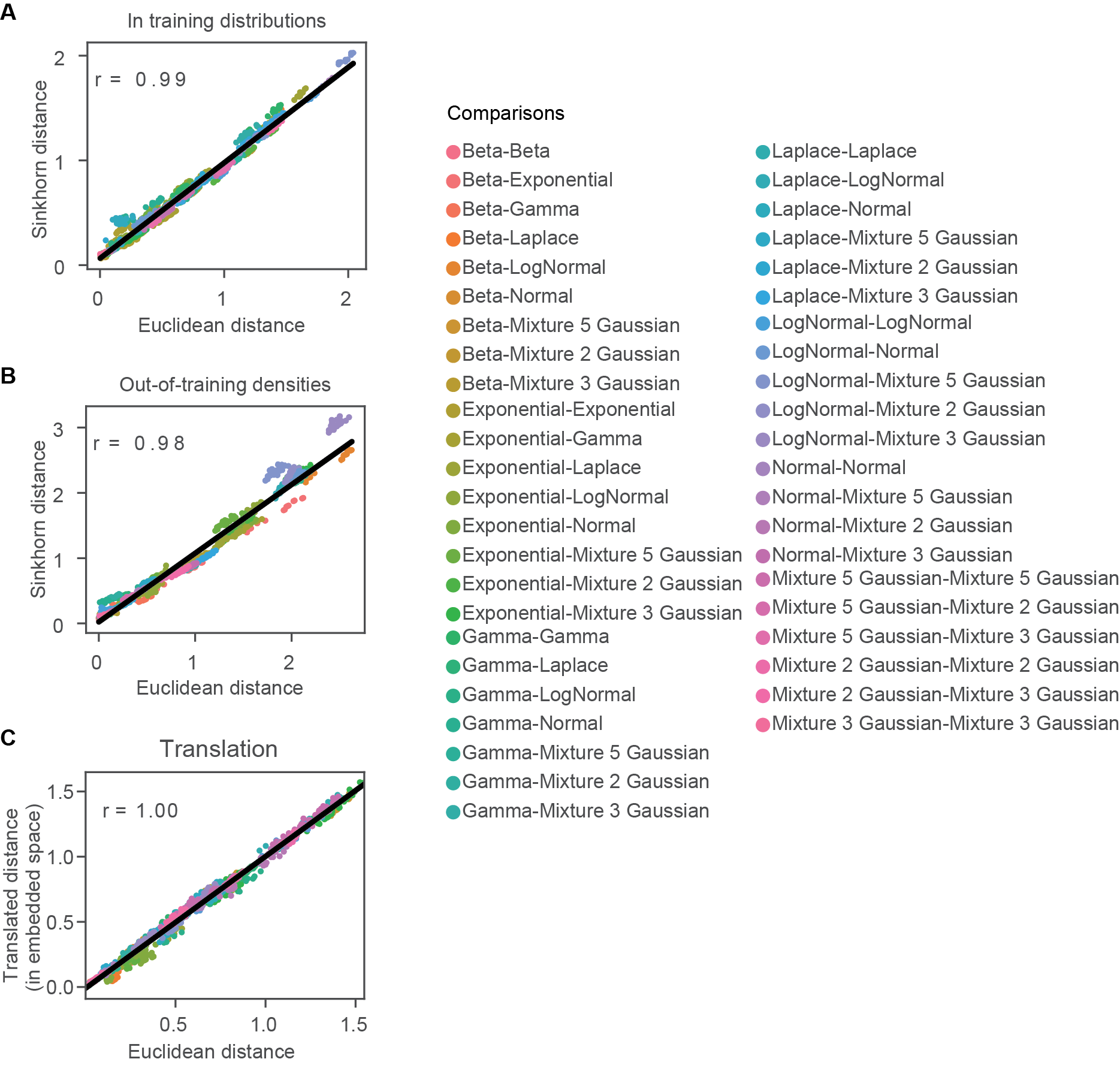}
  \caption[0.95\textwidth]{Pearson's r correlation coefficient for  association between embedded and Sinkhorn distances}
  \label{fig:corrplot}
\end{figure}

In this section we show some additional figures: The correlation plots colored by the densities (similar to Fig 2, in the main text, but with full legend). We also show our results for the same experiments discussed earlier with the $W_2$ metric (Fig~\ref{fig:w2_expt}. We found a strong correlation between the variances of the densities and the $y$-coordinates of the encoded points. However unlike in the $W_1$ case, we found no such relations between the means and the $x$-coordinates. Finally we show in Figure~\ref{fig:scale} how our network has learnt to respect the scaling law (\textbf({iii}) in Theorem~\ref{wass_prop}).  
\begin{figure}[h!]
\includegraphics[width=\linewidth]{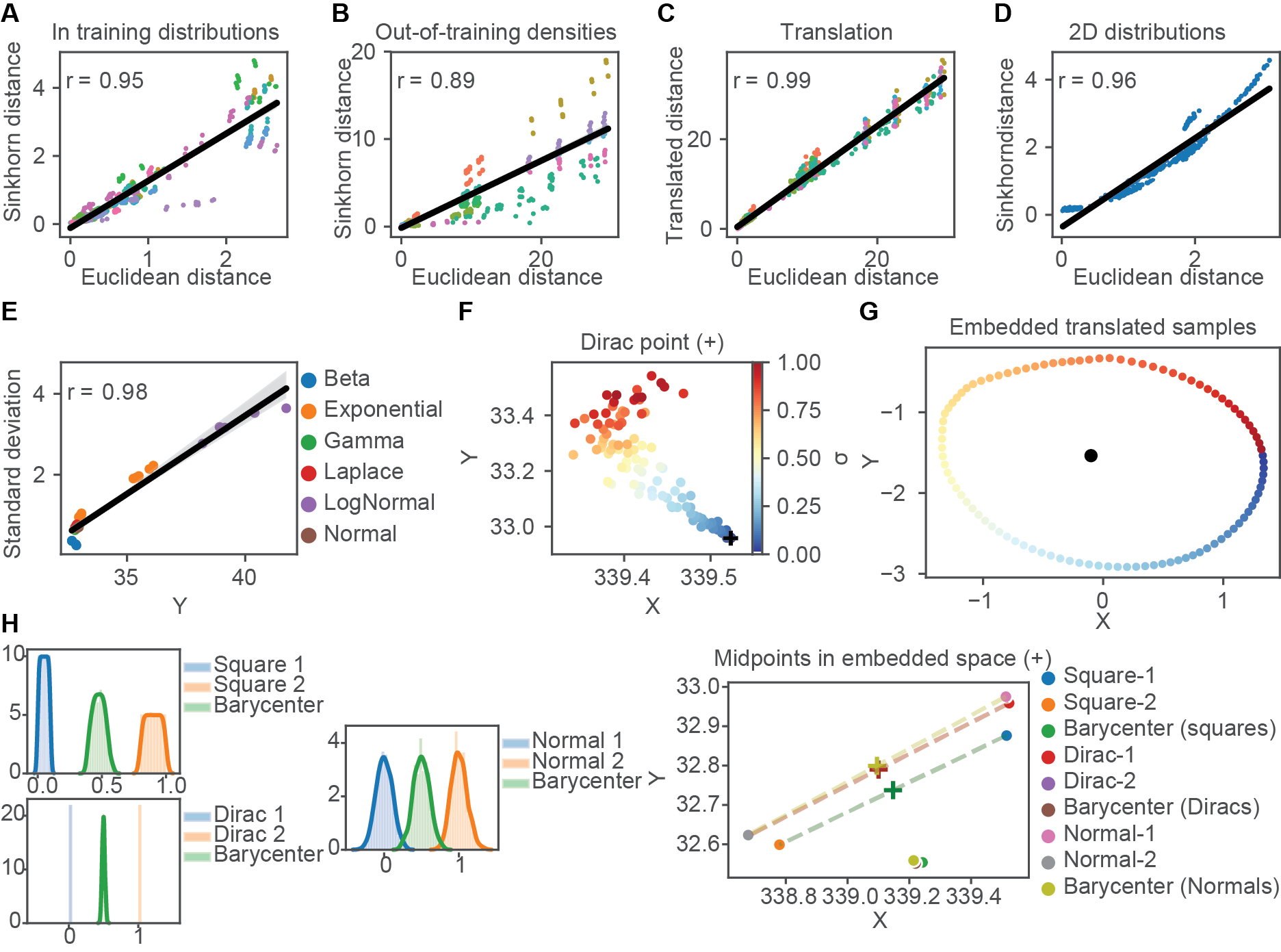}
  \caption[0.95\textwidth]{Experiments for network train to measure $W_2$. (\textbf{A}--\textbf{D}) Correlations with $W_2$ and embedded distances, for 1D distributions (\textbf{A}, \textbf{B}) under translations (\textbf{C}) and for 2D Normal distributions (\textbf{D}). (\textbf{E}--\textbf{F}) Interpretation of embedding axes showing Pearson's r correlation between standard deviation and y-axis (\textbf{E}) and convergence of samples from 1D Normal distributions with various standard deviations to encoded sample of the Dirac distribution (\textbf{F}). (\textbf{G}) Samples of 2D Normal distribution translated around a circle with black dot representing un-translated embedding. (\textbf{H}) Barycenters of distributions (left) and midpoints drawn between lines connecting the encoded samples (right) after training the network on $W_2$ distances.}
  \label{fig:w2_expt}
\end{figure}

\begin{figure}[h!]
\centering
\includegraphics[width=\linewidth]{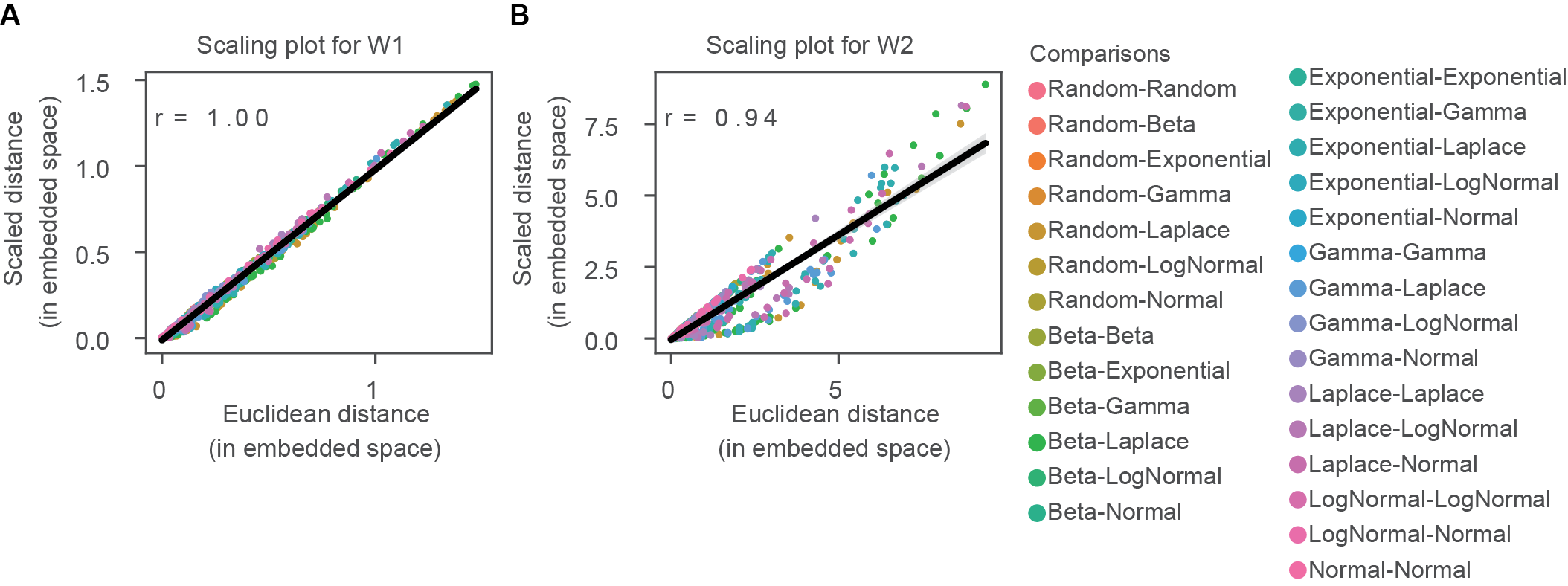}
\caption[0.95\textwidth]{Correlation after scaling; empirically validating property (\textbf{iii}) in Theorem~\ref{wass_prop}, (\textbf{A}) Network trained with $W_1$ metric, (\textbf{B}) Network trained with $W_2$ metric }
\label{fig:scale}
\end{figure}

\section{Optimal Transport}
 The optimization problem defining the distance (Equation~\ref{eqn:ot}) is popularly known as optimal transport or the Monge–Kantorovich problem. The Kantorovich formulation~\cite{kontorovich} of the transportation problem is:
\begin{equation}
\text{OT}(\mu, \nu) :=  \min_{\pi \in \Pi(\mu,\nu)} \int_{\mathcal{X}\times \mathcal{X}} c(x, y) d\pi(x, y)
\end{equation}
where $c(\cdot, \cdot) : \mathcal{X} \times \mathcal{X} \rightarrow \mathbb{R}$ is a cost function and the set of couplings $\Pi(\mu, \nu)$ consists of joint probability distributions over the product space $\mathcal{X} \times \mathcal{X}$
with marginals $\mu$ and $\nu$, 
\begin{equation}
\Pi(\mu,\nu) := \{\pi \in \mathbb{P}(\mathcal{X}\times \mathcal{X} ) : P_1\#\pi = \mu, P_2\#\pi = \nu\}. 
\end{equation}
where $P_i$ are the projection maps from $\mathcal{X}\times \mathcal{X}$ to $i$th factor of $\mathcal{X}$ and $P_i \# \pi$ is the pushforward of the measure $\pi$ onto $\mathcal{X}$. The cost function generally reflects the metric of the space $\mathcal{X}$ and in our case is just $c(x,y) := \lvert\rvert x-y\lvert\rvert_p^{1/p}$ for some $p \geq 1.$
However as noted in the main text, solving the above problem scales cubically on the sample sizes and is extremely difficult in practice. Adding an entropy regularization leads to a problem that can be solved much more efficiently~\cite{cuturi2013sinkhorn, altschuler2018nearlinear, peyre2020computational}. For the convenience of the reader, let us recall the entropy regularization as in~\cite{cuturi2013sinkhorn}.  We first construct discrete measures $\mu := \sum_{i=1}^{n}   a_i\delta_{x_{i}}$ and  $\nu := \sum_{j=1}^{m}b_j\delta_{y_{j}}$ where $a, b$ are vectors in the probability simplex, and let $C$ is the cost matrix given $c_{ij} := c(x_i,y_j)$, then the optimization problem can be succinctly written as \begin{equation}\label{eqn:discrete_ot}
    W_p(\mu, \nu) = \min_{P \in U(\mu,\nu)} \sum_{i,j} P_{ij}C_{ij}
\end{equation}
where $U(\mu, \nu) = \{P \in R^{n \times m}_{+}  : P\mathbf{1}_m = a , P^{t}\mathbf{1}_n = b \}$. \\
The entropy regularized version of this problem reads:
\begin{equation}\label{eqn:entropy_ot}
    SD^{\lambda}_p(\mu,\nu) := \min_{P \in U(\mu,\nu)} \sum_{i,j}P_{ij}C_{ij} +  \frac{1}{\lambda} \sum_{i,j}P_{ij}\text{log}P_{ij}
\end{equation}
Due to the strong convexity introduced by the regularizer, the above problem now has a unique solution and can be efficiently solved by the Sinkhorn algorithm. In our work $\lambda = 10$.
\section{Some technical considerations}
Note that $SD^{\lambda}_p$ is not a true metric as it do not satisfy $SD^{\lambda}_p(X,X) = 0$, for all sets of samples~\cite{cuturi2013sinkhorn}. However it is symmetric and satisfies the triangle inequalities. We circumvent this issue by only using continuous measures during training time. This ensures that any of two sets of samples drawn a given measure are distinct with probability $1$. Thus, during training we never encounter the set $X$ twice, so a case where $||H_{\theta}(X)- H_{\theta}(X)|| = SD^{\lambda}_p(X,X)$ never arises. Thus we end up learning a metric space where the distances between \textit{different} samples are approximately equal to the Sinkhorn distance.  

\section{Wasserstein Barycenters}
Given measures $\mu_1,\cdots, \mu_N $, we define the Wasserstein barycenter as the minimizer of the functional
\begin{equation}\label{eqn:barycenter}
F[\nu] = \sum_{i=1}^{N}w_i W^{p}_p((\nu, \mu_i )
\end{equation}
where $w_i$ are some fixed weights and $\sum_{i} w_i =1$. For simplicity, we will take the weights to be $1/N$. We use the algorithm in~\cite{cuturi2014fast}, as well as the Geomloss library to compute the barycenters. \\
Another way to view our experiments with the barycenters is to measure how closely we can approximate Wasserstein geodesics via straight lines in our encoded Euclidean space. The results show a better approximation under $W_1$ metric than the $W_2$ metric. We hypothesize since $SD^{\lambda}$ discretize the space, it can still approximate $W_1$ metric since $\mathbb{P}(\mathcal{X})$ is flat. However a straight line is poor approximation for geodesics in a non-flat space, which explains our poor results for the $W_2$ metric. \\ 
Finally our experiments with the barycenters suggest a natural way to embed measures in our $2$-dimensional encoded space. Take random samples of size $N$ and repeat this process $M$ times. Our encoder will take in these $M \times N$ samples and produce $M$ points. We can take the centroid of these points and use it to get a representation of our measure. 

\section{Quantitative Results} 
In this section we report our quantitative results of our models in various tasks as well as performance of our models without the regularizers. The numbers reported below are the RMSE errors between true distances and the our calculated distances. All the experiments were run $5$ times. 
\begin{table}[h]
\label{tab:results_1}
\caption{Results of our experiments on 1D measures}
\resizebox{\columnwidth}{!}{%
\begin{tabular}{@{}ccccccc@{}}
\toprule
Tasks & \begin{tabular}[c]{@{}c@{}} Model trained \\  with $W_1$ \end{tabular} & \begin{tabular}[c]{@{}c@{}}Model trained with $W_1$ \\ (no regs) \end{tabular} &  \begin{tabular}[c]{@{}c@{}} Model trained with $W_1$ \\ (scaling reg) \end{tabular} & \begin{tabular}[c]{@{}c@{}} Model trained \\ with $W_2$ \end{tabular} &  \begin{tabular}[c]{@{}c@{}} Model trained with $W_2$\\ (no regs) \end{tabular} &  \begin{tabular}[c]{@{}c@{}} Model trained with $W_2$ \\ (scaling reg) \end{tabular} \\
\midrule
In sample densities & .01 $\pm$ .007 & .08 $\pm$ .004 & .04 $\pm$ .005 & .06 $\pm$ .008 & .1 $\pm$ .003 & .08 $\pm$ .003 \\
Out-of-sample densities & .02 $\pm$ .009 & .1 $\pm$ .004 & .07 $\pm$ .001 & .1 $\pm$ .041 & .19 $\pm$ .03  & .17 $\pm$ .06 \\
Translation property & .01 $\pm$ .004 & .07 $\pm$ .006 & .05 $\pm$ .005 & .01 $\pm$ .008 & .09 $\pm$ .004 & .07 $\pm$ .001 \\
Scaling property & .01 $\pm$ .003 & .08 $\pm$ .006 & .03 $\pm$ .002 & .05 $\pm$ .001 & .12 $\pm$ .01 & .07 $\pm$ .009\\
Barycenter Accuracy & .05 $\pm$ .008 & .09 $\pm$ .007& .08 $\pm$ .009 & .09 $\pm$ .004  & .13 $\pm$ .05 & .11 $\pm$ .07 \\ 
\bottomrule
\end{tabular}
}%
\end{table}

\begin{table}[h]
\label{tab:results_2}
\caption{Results of our experiments on 2D measures}
\resizebox{\columnwidth}{!}{%
\begin{tabular}{@{}ccccccc@{}}
\toprule
Tasks & \begin{tabular}[c]{@{}c@{}} Model trained \\  with $W_1$ \end{tabular} & \begin{tabular}[c]{@{}c@{}}Model trained with $W_1$ \\ (no regs) \end{tabular} &  \begin{tabular}[c]{@{}c@{}} Model trained with $W_1$ \\ (scaling reg) \end{tabular} & \begin{tabular}[c]{@{}c@{}} Model trained \\ with $W_2$ \end{tabular} &  \begin{tabular}[c]{@{}c@{}} Model trained with $W_2$\\ (no regs) \end{tabular} &  \begin{tabular}[c]{@{}c@{}} Model trained with $W_2$ \\ (scaling reg) \end{tabular} \\
\midrule
In sample densities & .02 $\pm$ .005 & .1 $\pm$ .001 & .06 $\pm$ .003 & .03 $\pm$ .008 & .15 $\pm$ .01 & .09 $\pm$ .009 \\
Out-of-sample densities & .05 $\pm$ .004 & .19 $\pm$ .02 & .13 $\pm$ .04 & .08 $\pm$ .01 & .2 $\pm$ .01  & .18 $\pm$ .04 \\
Translation property & .09 $\pm$ .004 & .17 $\pm$ .06 & .15 $\pm$ .02 & .04 $\pm$ .008 & .19 $\pm$ .04 & .1 $\pm$ .04 \\
Scaling property & .07 $\pm$ .01 & .19 $\pm$ .02 & .14 $\pm$ .02 & .08 $\pm$ .009 & .21 $\pm$ .01 & .11 $\pm$ .07\\
\bottomrule
\end{tabular}
}
\end{table}
The results on translation and scaling properties are calculated on both in sample densities and out of sample densities. The performance on both the datasets are similar so we pooled all the results and reported the average scores and the standard deviations (10 trials, 5 with in-training densities and 5 with out-of-sample densities). 

\section{Sampling sizes} The size of the sample plays a crucial role here. What is the right size of samples to pick? If the size of the samples $X \sim \mu, Y \sim \nu $ are large then our method works well. But picking a large sample size is computationally very expensive. We found a sample size of $500$ yields good results while a sample size of below $100$ yields inconsistent results (variance is high). 
\end{document}